\documentclass[conference]{IEEEtran}
\usepackage{times}

\usepackage[numbers]{natbib}
\usepackage{multicol}
\usepackage{mathtools}

\usepackage[]{graphicx}
\usepackage{enumerate}
\usepackage{epstopdf}
\usepackage{subfigure}

\usepackage{amssymb}
\usepackage{amsthm}

\graphicspath{{./fig/}}
\usepackage{color}
\usepackage{amsmath}
\usepackage{amssymb}
\usepackage{graphicx}
\usepackage{comment,xspace}
\usepackage{hyperref}
\usepackage{fancybox}

\newtheorem{theorem}{Theorem}[section]

\newtheorem{lemma}[theorem]{Lemma}

\newcommand{\real}{{\mathbb{R}}}
\newcommand{\reals}{\real}

\renewcommand{\natural}{{\mathbb{N}}}
\newcommand{\naturals}{\natural}

\usepackage{algorithmic}
\usepackage{algorithm}

\title{Control of Robotic  Mobility-On-Demand Systems:  \\a Queueing-Theoretical Perspective}
\author{Rick Zhang and Marco Pavone}

\begin{document}
\maketitle

\begin{abstract}
In this paper we present and analyze a queueing-theoretical model for autonomous mobility-on-demand (MOD) systems where robotic, self-driving vehicles transport customers within an urban environment and rebalance themselves to ensure acceptable quality of service throughout the entire network. We cast  an autonomous MOD system within a closed Jackson network model with passenger loss. It is shown that an optimal rebalancing algorithm minimizing the number of (autonomously) rebalancing vehicles  and keeping vehicles availabilities balanced throughout the network can be found by solving a linear program.  The theoretical insights are used to design a robust, real-time rebalancing algorithm, which is applied to a case study of New York City. The case study  shows that the current taxi demand in Manhattan can be met with about 8,000 robotic vehicles (roughly 60\% of the size of the current taxi fleet). Finally, we extend our queueing-theoretical setup to include congestion effects, and we study the impact of autonomously rebalancing vehicles on overall congestion. Collectively, this paper provides a rigorous approach to the problem of system-wide coordination of autonomously driving vehicles, and provides one of the first characterizations of the sustainability benefits of robotic transportation networks.
\end{abstract}

\IEEEpeerreviewmaketitle

\section{Introduction}

According to United Nations estimates, urban population will double in the next 30 years \cite{UN:07}; 
for example, in the next 10 years, the population of Riyadh, Saudi Arabia, will double from 5 million to 10 million.
Given the limited \hbox{availability for} additional roads and parking spaces in current (mega)-cities,  private automobiles appear as an \emph{unsustainable} solution for the future of \emph{personal} urban mobility \cite{Mitchel.Bird.ea:10}. Arguably, \hbox{one of the} most promising approaches to cope with this problem is one-way vehicle sharing with electric cars (referred to as Mobility-On-Demand, or MOD), which directly targets the problems of parking spaces, pollution, and low vehicle utilization rates \cite{Mitchel.Bird.ea:10}. Limited-size MOD systems with human-driven vehicles have recently been deployed in several European and American cities \cite{C2G:11}. However, such systems lead to vehicle \emph{imbalances}, that is some stations become rapidly depleted of vehicles while others have too many, due to some stations being more popular than others. Somewhat surprisingly, even if the transportation network is symmetric (that is, the underlying network topology is a regular grid and arrival rates and routing choices of customers at all nodes are uniform), the \emph{stochastic} nature of customer arrivals to the stations will quickly drive the system out of balance and hence to instability (since the customer queue will grow without bound at some stations) \cite{Fricker2012}. 

The related problem of rebalancing in the increasingly popular bike-sharing systems is solved using trucks which can carry many bikes at the same time, and algorithms have very recently been developed to optimize the truck routes \cite{DellAmico2013, DiGaspero2013}. Since this approach is not feasible with cars, the work in \cite{Smith2013} considers the possibility of hiring a team of rebalancing drivers whose job is to rebalance the vehicles throughout the transportation network. However, with this approach the rebalancing drivers themselves become unbalanced, and one needs to  ``rebalance the rebalancers,'' which significantly increases congestion and costs  \cite{Smith2013}. 
Another option would be to incentivize ride sharing \cite{Barth2001}, which, unfortunately, defeats the purpose of a MOD system to ensure \emph{personal} mobility.

Recently, a transformational technology has been proposed in \cite{Pavone.ea:IJRR12, Burns2013}, whereby \emph{driverless} electric cars shared by the customers  provide on-demand mobility. Autonomous driving holds great promise for MOD systems because robotic vehicles can rebalance themselves (thus eliminating the rebalancing problem at its core), enable system-wide coordination, free passengers from the task of driving, and \emph{potentially} increase safety.
Indeed, robotic vehicles specifically designed for personal urban mobility are already being marketed (e.g., the Induct Navia vehicle \cite{Navia}, the General Motors' EN-V  vehicle \cite{GMENV}, and the Google car \cite{Fisher2013}). Yet, little is known about how to design and operate robotic transportation networks \cite{Pavone.ea:IJRR12}.

\emph{Statement of contributions}:  The objective of this paper is to develop a model of, study rebalancing algorithms for, and evaluate the potential benefits of a MOD system where mobility is provided by driverless cars (henceforth referred to as autonomous MOD system). Rebalancing algorithms for autonomous MOD systems have been investigated in \cite{Pavone.ea:IJRR12} under a fluidic approximation  (i.e., customers and vehicles are modeled as a continuum). While this approach provides valuable insights for the operation of an autonomous MOD system, by its very nature, it does not provide information about the effect of stochastic fluctuations in the system (e.g., due to the customers' arrival process) and, most importantly, it does not allow the computation of key performance metrics such as availability of vehicles at stations and customer waiting times. This motivates the queueing-theoretical approach considered in this paper. In this respect, our work is related to \cite{George2011, Waserhole2013}, where a transportation network comprising traditional (i.e., human-driven) shared vehicles is modeled within the framework of Jackson networks \cite{Serfozo1999}. The key technical difference  is that in this paper we address the problem of \emph{synthesizing} a rebalancing policy, rather than \emph{analyzing} the evolution of the vehicle distribution under the customers' routing choices.

Specifically, the contribution of this paper is fourfold. First, we propose a queueing-theoretical model of an autonomous MOD system cast within a Jackson network model. Second, we study the problem of synthesizing rebalancing algorithms,  where the control objective is to minimize the number of (autonomously) rebalancing vehicles on the roads while keeping vehicle availabilities balanced throughout the network. Remarkably, we show that under certain assumptions an optimal policy can be solved as a linear program. Third, we  apply our theoretical results to a case study of New York City, which shows that the current taxi demand in Manhattan can be met with about 8,000 robotic vehicles (roughly 60\% of the size of the current taxi fleet). This shows the potential of autonomous MOD systems. Finally, by leveraging our queueing-theoretical setup, we study the potential detrimental effect of rebalancing on traffic congestion (rebalancing vehicles, in fact, \emph{increase} the number of vehicles on the roads). Our study suggests that  while autonomously rebalancing vehicles can have a detrimental impact on traffic congestion in already-congested systems, in most cases this is not generally a concern as rebalancing vehicles ``tend" to travel along less congested roads. To the best of our knowledge, this is the first paper to provide a rigorous, stochastic approach to the problem of system-wide coordination of autonomously driving vehicles.

\emph{Organization}:  The remainder of the paper is structured as follows: In Section \ref{sec:back} we briefly review some well-known results of queueing networks, specifically Jackson networks. In Section \ref{sec:model} we show how to model an autonomous MOD system with rebalancing within a Jackson network model. In Section \ref{sec:rebalancing} we formulate the optimal rebalancing problem, we show that it can be solved via a linear program, we provide an iterative algorithm to compute relevant performance metrics (chiefly, vehicle availability at stations), and we use the theoretical insights to design a robust, real-time rebalancing policy. In Section \ref{sec:fleet} we apply our model and algorithms to a case study of New York City, while in Section \ref{sec:cong} we extend our queueing-theoretical setup to include congestion effects. Finally, in Section \ref{sec:conc} we draw our conclusions, and present directions for future research.

\vspace{-0.3em}
\section{Background Material}\label{sec:back}
In this section we review some key results from the theory of Jackson networks,  on which we will rely extensively later in the paper. 
Consider a network consisting of $\lvert \mathcal{N} \rvert$ first-come first-serve nodes, or queues, where $\mathcal{N}$ represents the set of nodes in the network. Discrete customers arrive from outside the network according to a stochastic process or move among the nodes. Customers that arrive at each node are serviced by the node, and proceed to another node or leave the system. A network is called \emph{closed} if the number of customers in the system remains constant and no customers enter or leave the network. A Jackson network is a Markov process where customers move from node to node according to a stationary routing distribution $r_{ij}$ and the service rate $\mu_i(n)$ at each node $i$ depends only on the number of customers at that node, $n$ \cite[p.9]{Serfozo1999}. For the remainder of this paper, we consider only closed networks. The state space of a closed Jackson network with $m$ customers is given by
${\Omega_m = \Big\{ x = (x_1, x_2, ... , x_{\lvert \mathcal{N} \rvert} ) : \sum_{i = 1}^{\lvert \mathcal{N} \rvert} x_i = m, x_i \in \mathbb{Z}_{\geq 0} \Big\}}$,
where $x_i$ is the number of customers at node $i$. Jackson networks are known to admit a product-form stationary distribution, where the stationary distribution of the network is given by a product of the distribution of each node. In equilibrium, the arrival rates (or throughput) at each node satisfy the traffic equations
\begin{equation}
\pi_i = \sum_{j \in \mathcal{N}} \pi_j r_{ji} \;\;\;\; \forall i \in \mathcal{N}.
\label{eq:traffic}
\end{equation}

For a closed network, equation \eqref{eq:traffic} does not have a unique solution, and $\pi = (\pi_1\ \pi_2\ ...\ \pi_N)^T$ only determines the arrival rates up to a constant factor, and hence are sometimes called the relative throughput.  The stationary distribution of the network is given by 
\begin{equation*}
\mathbb{P}(x_1, x_2, ..., x_{\lvert \mathcal{N} \rvert}) = \frac{1}{G(m)} \prod_{j = 1}^{\lvert \mathcal{N} \rvert} \pi^{x_j}_j \prod_{n=1}^{x_j} \mu_j(n)^{-1}.
\end{equation*}

The quantity $G(m)$ is the normalization constant needed to make $\mathbb{P}(x_1, x_2, ..., x_{\lvert \mathcal{N} \rvert})$ a probability measure, and is given by $G(m) = \sum_{x \in \Omega_m} \prod_{j = 1}^{\lvert \mathcal{N} \rvert} \pi^{x_j}_j \prod_{n=1}^{x_j} \mu_j(n)^{-1}$. Many performance measures of closed Jackson networks can be expressed in terms of the normalization factor $G(m)$. In \cite[p.27]{Serfozo1999}, it is shown that the actual throughput of each node (average number of customers moving through node $i$ per unit time) is given by 
\begin{equation}
\Lambda_i(m) = \pi_i \, G(m-1)/G(m).
\label{eq:throughput}
\end{equation}
One can further define the quantity 
\begin{equation}
\gamma_i = \pi_i/\mu_i(1) \;\;\;\; \forall i \in \mathcal{N},
\label{eq:gamma}
\end{equation}
where $\gamma_i$ is referred to as the relative utilization of node $i$. 
Lavenberg \cite[p.128]{Lavenberg1983} showed that the marginal distribution of the queue length variable $X_i$ at node $i \in \mathcal{N}$ is given by
\begin{equation*}
\mathbb{P}(X_i = x_i) = \gamma_i^{x_i}[G(m-x_i) - \gamma_i G(m - x_i - 1)]/G(m).
\end{equation*}

A quantity that will be useful is the probability that a node has at least 1 customer, which we refer to as $A_i(m)$. This is given by
\vspace{-0.3em}
\begin{align}
A_i(m) &= 1 - P(X_i = 0)  \notag \\
&= 1 - \frac{G(m)- \gamma_i G(m-1)}{G(m)} = \frac{\gamma_i G(m-1)}{G(m)}.
\label{eq:Ai}
\end{align}
\section{Model Description and Problem Formulation} \label{sec:model}

\subsection{Model of autonomous MOD system}
In this paper, we model an autonomous MOD system within a queueing theoretical framework. Consider $N$ stations placed within a given geographical area and $m$ (autonomous) vehicles that provide service to customers. Customers arrive at each station $i$ according to a time-invariant Poisson process with rate $\lambda_i \in \reals_{>0}$. Upon arrival, a customer at station $i$ selects a destination $j$ with probability $p_{ij}$, where $p_{ij}\in \reals_{\geq 0}$, $p_{ii}=0$, and  $\sum_j p_{ij} = 1$. Furthermore, we assume that the probabilities $\{p_{ij}\}_{ij}$ constitute an irreducible Markov chain. If there are vehicles parked at station $i$, the customer takes the vehicle and travels to her/his selected destination. Instead, if the station is empty of vehicles, the customer immediately leaves the system. This type of customer model will be referred to as a ``passenger loss'' model (as opposed to a model where customers form a queue at each station).  A consequence of the passenger loss model is that the number of passengers at each station at a fixed instant in time is 0 (since passengers either depart immediately with a vehicle or leave the system). We assume that each station has sufficiently many parking spaces so that vehicles can always immediately park upon arrival at a station. The travel time from station $i$ to station $j$ is an exponentially distributed random variable with mean equal to $T_{ij}\in \reals_{> 0}$. The travels times for the different customers are assumed to constitute an independently and identically distributed sequence (i.i.d.). The vehicles can \emph{autonomously}
travel throughout the network in order to rebalance themselves and best anticipate future demand. The performance criterion that we are interested in is the availability of vehicles at each station (or conversely the probability that a customer will be lost). 

A few comments are in order. First, our model captures well the setup with impatient customers, not willing to make use of a MOD system if waiting is required. In this respect, our model appears suitable to study the benefits of autonomous MOD systems whenever high quality of service (as measured in terms of average waiting times for available vehicles) is required. From a practical standpoint, the loss model assumption significantly simplifies the problem, as it essentially allows us to decouple the ``vehicle process'' and the ``customer process'' (see Section \ref{sec:network}). Second, travel times, in practice, do not follow an exponential distribution. However we make this assumption as (i) it simplifies the problem considerably, and (ii) reasonable deviations from this assumption have been found not to alter in any practical way the predictive accuracy of similar spatial queuing models used for vehicle routing \cite{larson1981urban}. Third, the assumption that the probabilities $\{p_{ij}\}_{ij}$ constitute an irreducible Markov chain appear appropriate for dense urban environments. Finally, our model does not consider congestion, which is clearly a critical aspect for the efficient system-wide coordination of autonomous vehicles in a MOD system. The inclusion of congestion effects will be discussed in Section  \ref{sec:cong}.

\subsection{Casting an autonomous MOD system into a Jackson model} \label{sec:network}

The key idea to cast an autonomous MOD system into a Jackson model is to consider an \emph{abstract queueing network} where we identify the stations with single-server (SS) nodes (also referred to as ``station'' nodes) and the roads with infinite-server (IS) nodes (also referred to as ``road'' nodes). Assume, first, the simplified scenario where vehicles do not perform rebalancing trips (in which case, the model is essentially identical to the one in \cite{George2011}). In this case, at each station node, vehicles form a queue while waiting for customers and are ``serviced'' when a passenger arrives. A vehicle departing from a SS node moves to the IS node that connects the origin to the destination selected by the customer. After spending an exponentially distributed amount of time in the IS node (i.e., the ``travel time''), the vehicle moves to the destination SS node. According to our model, once a vehicle leaves a SS (station) node $i$, the probability that it moves to the IS (road) node $ij$ is $p_{ij}$. The vehicle then moves to SS (station) node $j$ with probability 1. Note that with this identification we have modeled a MOD system (at least in the case without rebalancing)  as a \emph{closed queueing network with respect to the vehicles}. Note that the road queues are modeled as infinite-server queues as the model does not consider congestion effect (in Section \ref{sec:cong} we will see that if congestion is taken into account the road queues become \emph{finite}-server queues).

More formally, denote by $S$ the set of single-server nodes  and $I$ the set of infinite-server nodes. Each station is mapped  into an SS node, while each road is mapped into an IS node. The set of all nodes in the abstract queueing network is then given by  $\mathcal{N} = S \cup I$. Since each SS node is connected to every other SS node, and since $p_{ii} = 0$ (hence the road node $ii$ does not need to be represented), the number of nodes in the network is given by $N(N-1) + N  =N^2$, in other words, $\lvert\mathcal{N}\rvert = N^2$. 
For each IS node $i \in I$, let $\text{Parent}(i)$ and $\text{Child}(i)$ be the origin and destination nodes of $i$, respectively. As explained before, vehicles in the abstract queueing network move between SS nodes as IS nodes according to the routing matrix  $\{r_{ij}\}_{ij}$:
\begin{equation}
r_{ij} =
\begin{cases}
	p_{il} &  i \in S, j \in I \text{ where } i = \text{Parent}(j), l = \text{Child}(j), \\
	1 &  i \in I, j \in S \text{ where } j = \text{Child}(i), \\
	0 & \text{otherwise},
\end{cases}
\label{eq:rij}
\end{equation}
where the first case corresponds to a move from a SS node to an IS node (according to the destination selected by a customer), and the second case to a move from an IS node to the unique SS node corresponding to its destination. Furthermore, the service times at each node $i \in \mathcal N$ are exponentially distributed with service rates given by
\begin{equation}
\mu_{i}(n) = 
\begin{cases}
	\lambda_{i} & \text{if } i \in S, \\
	n \cdot \mu_{jk} & \text{if } i \in I, j = \text{Parent}(i), k = \text{Child}(i),
\end{cases}
\label{eq:muij}
\end{equation}
where $n \in\{0, 1, \ldots\, m\}$ is the number of vehicles at node $i$, and  $\mu_{jk} = 1/T_{jk}$. The first case is the case where vehicles wait for customers at stations, while the second case is the case where vehicles spend an exponentially distributed travel time to move between stations (note that the IS nodes correspond to infinite-server queues, hence the service rate is proportional to the number of vehicles in the queue). As defined, the abstract queuing network is a closed Jackson network, and hence can be analyzed with the tools discussed in Section \ref{sec:back}.

Assume, now, that we allow the vehicles to autonomously rebalance throughout the network. To include rebalancing while staying within the Jackson network framework, we focus on a particular class of stochastic rebalancing policies described as follows. Each station $i$ generates ``virtual passengers'' according to a Poisson process with rate $\psi_i$, \emph{independent} of the real passenger arrival process, and routes these virtual passengers to station $j$ with probability $\alpha_{ij}$ (with $\sum_j \alpha_{ij} =1$ and $\alpha_{ii} = 0$). As with real passengers, the virtual passengers are lost if the station is empty upon arrival. Such class of  rebalancing policies \emph{encourages} rebalancing but does \emph{not} enforce a rebalancing rate, which allows us to maintain tractability in the model. 

One can then combine the real passenger arrival process with the virtual passenger process (assumed independent) using the independence assumption to form a model of the same form as the one described in Section \ref{sec:network} while taking into account vehicle rebalancing.
Specifically, we consider the same set of SS nodes and IS nodes (since the transportation network is still the same). Let $\{A^{(i)}_t, t \geq 0\}$ be the total arrival process of real \emph{and} virtual passengers at station $i \in S$, and denote its rate with $\tilde{\lambda}_i$. The process  $A^{(i)}_t$ is Poisson since it is the superposition of two independent Poisson processes. Hence, the rate $\tilde{\lambda}_i$ is given by
\vspace{-0.2em}
\begin{equation}\label{eq:lgen}
\tilde{\lambda}_i = \lambda_i + \psi_i.
\vspace{-0.2em}
\end{equation}
Equivalently, one can view the passenger arrival process and the rebalancing process as the result of Bernoulli splitting on $A^{(i)}_t$ with a probability $p_i$ satisfying 
\vspace{-0.2em}
\begin{equation}\label{eq:Bernoulli}
\psi_i = p_i \tilde{\lambda}_i, \quad \lambda_i = (1-p_i) \tilde{\lambda}_i. 
\vspace{-0.2em}
\end{equation}
Let us refer to passenger arriving according the the processes $\{A^{(i)}_t, t \geq 0\}$ as generalized passengers. The probability  $\tilde{p}_{ij}$ that  a generalized passenger arriving at station $i$ selects a destination $j$ is given by
\begin{spreadlines}{-0.05em}
\begin{align}
\tilde{p}_{ij} &= \mathbb{P}(i \rightarrow j \mid \text{virtual})\, p_i + \mathbb{P}(i \rightarrow j \mid \neg \text{virtual})\, (1 - p_i) \nonumber \\ 
&= \alpha_{ij} p_i + p_{ij} (1-p_i),
\label{eq:prob}
\end{align}
\end{spreadlines}
where $\mathbb{P}(i \rightarrow j \mid \text{virtual})$ is the probability of a virtual passenger to select station $j$ as its destination, and $\mathbb{P}(i \rightarrow j \mid \neg \text{virtual})$ is the probability of a real passenger to select station $j$ as its destination. One can then identify an autonomous MOD system with rebalancing (for the specific class of rebalancing policies discussed above) with an abstract queueing network with routing matrix and service rates given, respectively, by equations  \eqref{eq:rij} and \eqref{eq:muij}, where $p_{il}$ is replaced by $\tilde{p}_{il}$, $\lambda_i$ is replaced by $\tilde{\lambda}_i$, and $r_{ij}$ is replaced by $\tilde{r}_{ij}$.  In this way, the model is still a closed Jackson network model.  For notational convenience, we order $\gamma_i$ and $\pi_i$ (as defined in Section \ref{sec:back}) in such a way that the first $N$ components correspond to the $N$ stations (for example, $\gamma_i$ corresponds to station $i$, or the $i$th SS node, where $i = 1,2,...N$).

As already mentioned, in order to identify an autonomous MOD system with rebalancing with a Jackson queueing model, we restricted the class of rebalancing policies to open-loop, ``rebalancing promoting" policies. 
We will consider \emph{closed-loop} policies  in  Section \ref{sec:realtime}.

\subsection{Problem formulation}
Within our model, the optimization variables are the rebalancing rates  $\psi_i$ and $\alpha_{ij}$ of  the rebalancing promoting policies. One might wonder in the first place if and when rebalancing is even required. Indeed, one can easily obtain that, for the case \emph{without} rebalancing \cite{George2011},
$
 \lim_{m \rightarrow \infty} A_i(m) = \gamma_i/\gamma_S^{\text{max}}$, for all $i \in S$,
where $\gamma_i$ is the relative utilization of node $i\in S$ (see Section \ref{sec:back}) and $\gamma_S^{\text{max}} := \text{max}_{i \in S}\,  \gamma_i$. Hence, as $m$ approaches infinity, the set of stations $B := \{i \in S: \gamma_i = \gamma_S^{\text{max}}\}$ can have availability arbitrarily close to 1 while all other stations have availability \emph{strictly} less than 1 \emph{regardless} of $m$. In other words, without rebalancing, a MOD system will always experience customer losses no matter how many vehicles are employed!

The above discussion motivates the need for rebalancing. The tenet of our approach is to ensure, through rebalancing, that the network is (on average) in balance, i.e., $A_i(m)=A_j(m)$ for all $i,j \in S$ (or, equivalently, $\gamma_i = \gamma_j$ for all $i,j \in S$, as implied by equation  \eqref{eq:Ai}). 
The motivation behind this philosophy is twofold (i) it provides a natural notion of service fairness, and (ii) it fulfills the intuitive condition that as $m$ goes to $+\infty$ the availability of \emph{each station} goes to one (since in this case  $\gamma_i = \gamma_S^{\text{max}}$ for all $i$ in $S$). The objective then is to manipulate the rebalancing rates $\alpha_{ij}$ and $\psi_{i}$ such that all the $\gamma_i$'s in $S$ are equal while minimizing the number of rebalancing vehicles on the road. Note that the average number of rebalancing vehicles travelling between station nodes $i$ and $j$ is given by $T_{ij} \alpha_{ij} \psi_i$. The rebalancing problem we wish to solve is then as follows:
\begin{quote}{\bf Optimal Rebalancing Problem (ORP)}: Given an autonomous MOD system modeled as a closed Jackson network, solve
\begin{spreadlines}{-0.1em}
\begin{align}
&\underset{\psi_i, \alpha_{ij}}{\text{minimize}} && \sum_{i,j} T_{ij} \alpha_{ij} \psi_i \label{problem1} \\
&\text{subject to} && \gamma_i = \gamma_j  \notag  \\
& && \sum_j \alpha_{ij} = 1 \notag \\
& && \alpha_{ij} \geq 0, \, \, \psi_i \geq 0\qquad i,j\in \{1, \ldots, N\} \;\;\;\;\ \notag
\end{align}
\end{spreadlines}
where $\gamma_i = \frac{\pi_i}{\lambda_i + \psi_i}$ and $\pi_i$ satisfies equation \eqref{eq:traffic}. 
\end{quote}

Note that to solve the ORP one would need to explicitly compute the relative throughputs $\pi$'s using the traffic equation \eqref{eq:traffic}. This involves finding the 1-dimensional null space of a $\mathbb{R}^{N^2 \times N^2}$ matrix, which becomes computationally expensive as the number of stations become large. Furthermore, the objective function and the constraints $\gamma_i = \gamma_j$ are nonlinear in the optimization variables. 
In the next section we show how to reduce the dimension of the problem to $\mathbb{R}^N$ and how the ORP can be readily solved as a minimum cost flow problem. 

\section{Optimal Rebalancing} \label{sec:rebalancing}
\subsection{Optimal rebalancing}
In this section, we show how the ORP can be readily solved as a minimum cost flow problem. To this purpose, we first present two key lemmas, whose proofs are provided in the supplemental material. The first lemma shows how the traffic equations \eqref{eq:traffic} can be written only in terms of the SS nodes. 
\begin{lemma}[Folding of traffic equations]\label{lemma:folding}
Consider an autonomous MOD system modeled as a closed Jackson network as described in Section \ref{sec:network}. Then the relative throughputs $\pi$'s for the SS nodes can be found by solving the \emph{reduced} traffic equations
\begin{equation}
\vspace{-0.5em}
\pi_i = \sum_{k \in S} \pi_k \tilde{p}_{ki} \;\;\;\; \forall i \in S,
\label{eq:traffic2}
\end{equation}
where SS nodes are considered in isolation. The $\pi$'s for the IS nodes are then given by 
\vspace{-0.5em}
\begin{equation}
\pi_i = \pi_{\text{Parent}(i)} \tilde{p}_{\text{Parent}(i) \text{Child}(i)} \;\;\;\; \forall i \in I.
\label{eq:roadqueues}
\end{equation}
\end{lemma}
\begin{lemma}\label{lemma:gamma}
For any rebalancing policy $\{\psi_i\}_i$ and $\{\alpha_{ij}\}_{ij}$, it holds for all $i\in S$
\begin{enumerate}
\item $\gamma_i >0$,
\item  $ (\lambda_i + \psi_i) \gamma_i = \sum_{j \in S} \gamma_j (\alpha_{ji} \psi_j + p_{ji} \lambda_j)$.
\end{enumerate}
\end{lemma}
The next theorem (which represents the main result of this section) shows that we can \emph{always} solve problem ORP by solving a low dimensional linear optimization problem.
\begin{theorem}[Solution to problem ORP]\label{theo:main}
Consider the  linear optimization problem
\begin{spreadlines}{0.05em}
\begin{align}
&\underset{\beta_{ij}}{\emph{minimize}} && \sum_{i,j} T_{ij} \beta_{ij} \label{eq:opt2}\\
&\emph{subject to} && \sum_{j \neq i} (\beta_{ij} - \beta_{ji}) = -\lambda_i + \sum_{j \neq i} p_{ji} \lambda_j \notag \\
& && \beta_{ij} \geq 0 \notag
\end{align}
\end{spreadlines}
The optimization problem \eqref{eq:opt2} is always feasible. Let  $\{\beta_{ij}^*\}_{ij}$ denote an optimal solution.
By setting $\psi_i  = \sum_{j\neq i} \beta^*_{ij}$, $\alpha_{ii}=0$, and, for $j\neq i$,
\vspace{-0.5em}
\[
\alpha_{ij} =
\begin{cases}
\beta^*_{ij}/\psi_i& \textrm{\emph{if} } \psi_i > 0 \\ % \mathcal{S} is a state, not a set
1/{(N-1)} & \textrm{\emph{otherwise}},\\
\end{cases}
\]
\vspace{-0.5em}
one obtains an optimal solution to problem ORP.
\end{theorem}
\begin{proof}
First, we note that problem \eqref{eq:opt2} is an uncapacitated minimum cost flow problem and thus is always feasible. Consider an optimal solution to problem \eqref{eq:opt2}, $\{\beta^*_{ij}\}_{ij}$, and set $ \{\psi_i\}_i$ and $\{\alpha_{ij}\}_{ij}$ as in the statement of the theorem.
We want to show that, with this choice, $\{\psi_i\}_i$ and $\{\alpha_{ij}\}_{ij}$ represent an optimal solution to the ORP. Since $\{\beta^*_{ij}\}_{ij}$ is an optimal solution to problem \eqref{eq:opt2}, then one easily concludes that $\{\psi_i\}_i$ and $\{\alpha_{ij}\}_{ij}$ are an optimal solution to problem 
\vspace{-0.5em}
\begin{spreadlines}{-0.1em}
\begin{align}
&\underset{\psi_i, \alpha_{ij}}{\text{minimize}} && \sum_{i,j} T_{ij} \alpha_{ij}\psi_i \label{eq:opt3}\\
&\text{subject to} &&\lambda_i + \psi_i = \sum_{j} \alpha_{ji} \psi_j + p_{ji} \lambda_j \notag \\
& &&\sum_j \alpha_{ij} = 1 \notag\\
& && \alpha_{ij} \geq 0, \quad \psi_{i} \geq 0 \notag
\end{align}
\end{spreadlines}
The objective is now to show that the constraint
\vspace{-0.5em}
\begin{equation}\label{eq:con1}
\lambda_i + \psi_i = \sum_{j} \alpha_{ji} \psi_j + p_{ji} \lambda_j
\vspace{-0.5em}
\end{equation}
is equivalent to the constraint
\vspace{-0.5em}
\begin{equation}\label{eq:con2}
\gamma_i = \gamma_j.
\vspace{-0.5em}
\end{equation}
Consider, first, the case where the $\{\alpha_{ij}\}_{ij}$ and $\{\psi_i\}_i$ satisfy constraint \eqref{eq:con1}. Then, considering Lemma \ref{lemma:gamma}, one can write, for all $i$,
\vspace{-0.5em}
\begin{equation}\label{eq:inter}
\bigl( \sum_{j} \alpha_{ji} \psi_j + p_{ji} \lambda_j\Bigr) \gamma_i = \sum_{j \in S} \gamma_j (\alpha_{ji} \psi_j + p_{ji} \lambda_j).
\vspace{-0.5em}
\end{equation}
Let $\varphi_{ij}:= \alpha_{ji} \psi_j + p_{ji} \lambda_j$ and $\zeta_{ij}:= \varphi_{ij}/\sum_j \varphi_{ij}$. (Note that $\sum_j \varphi_{ij} = \lambda_i + \psi_i>0$ as $\lambda_i>0$ by assumption.) Since $\alpha_{ii} =0$ and $p_{ii}=0$, one has $\zeta_{ii} = 0$. The variables $\{\zeta_{ij}\}_{ij}$'s can be considered as transition probabilities of an \emph{irreducible} Markov chain (since, by assumption, the probabilities  $\{p_{ij}\}_{ij}$ constitute an irreducible Markov chain). Then, one can rewrite equation \eqref{eq:inter} as $\gamma_i = \sum_j \, \gamma_j \, \zeta_{ij}$, which can be rewritten in compact form as $Z\, \gamma = \gamma$, where $\gamma = (\gamma_1, \ldots, \gamma_N)^{T}$ and $Z$ is an irreducible, row stochastic matrix whose $i$th row is given by $[  \zeta_{i1}, \zeta_{i2}, \ldots, \zeta_{i \, i-1}, 0, \zeta_{i \, i+1} \ldots \zeta_{ i N}]$, where $i=1,2,\ldots, N$.
Since $Z$ is an irreducible, row stochastic matrix, by the Perron-Frobenius theorem \cite[p.673]{Meyer2000}, the eigenspace associated with the eigenvalue equal to 1 is one-dimensional, which implies that the equation $Z\, \gamma = \gamma$ has a unique solution given by $\gamma = (1,\ldots, 1)^T$, up to a scaling factor. This shows that $\gamma_i = \gamma_j$ for all $i, j$.   Conversely, assume that  $\{\alpha_{ij}\}_{ij}$ and $\{\psi_i\}_i$ satisfy constraint \eqref{eq:con2}. Considering Lemma \ref{lemma:gamma} (note, in particular, that $\gamma_i>0$), since $\gamma_i = \gamma_j$ for all $i,j$, then one immediately obtains that $\{\alpha_{ij}\}_{ij}$ and $\{\psi_i\}_i$ satisfy constraint \eqref{eq:con1}.
Hence, we can equivalently restate problem \eqref{eq:opt3} as problem \eqref{problem1},
which proves the claim.
\end{proof}
Remarkably, problem \eqref{eq:opt2} has the same form as the linear optimization problem in \cite{Pavone.ea:IJRR12}  used to find rebalancing policies within a \emph{fluidic} model of an autonomous MOD system. In this respect, the analysis of our paper provides a \emph{theoretical foundation} for the fluidic approximation performed in \cite{Pavone.ea:IJRR12}.

The importance of theorem \ref{theo:main} is twofold: it allows us to efficiently find an optimal open-loop, rebalancing promoting policy, and it enables the computation of quality of service metrics (namely, vehicle availability) for autonomous MOD systems as shown next.

\subsection{Computation of performance metrics}\label{subsec:metrics}

By leveraging Theorem \ref{theo:main}, one can readily compute performance metrics (i.e. vehicle availability) for an autonomous MOD system. First, we compute an optimal solution to the ORP using Theorem \ref{theo:main}, which involves solving a linear optimization problem with $N^2$ variables. Next, we compute the relative throughputs $\pi$'s using Lemma \ref{lemma:folding}. 
Finally, we apply a well-known technique called mean value analysis (MVA) \cite{Reiser1980} in order to avoid the explicit computation of the normalization constant in equation \eqref{eq:Ai}, which is prohibitively expensive for large numbers of vehicles and stations. The MVA algorithm is an iterative algorithm to calculate the mean waiting times $W_i(n)$ and the mean queue lengths $L_i(n)$ at each node $i$ of a closed separable system of queues, where $n = 1,2,\ldots$ is the numbers of customers over which the algorithm iterates. For the Jackson model in Section \ref{sec:network}, subject to the initial conditions $W_i(0) = L_i(0) = 0$, the equations for MVA read as (note that in our case, the customers of the abstract queueing systems are the vehicles, whose total number is $m$):
\begin{itemize}
\item $W_i(n) = \frac{1}{\mu_i(1)} = T_{\text{Parent}(i)\, \text{Child}(i)}\;\;\;\; \forall i \in I$,
\item  $W_i(n) = \frac{1}{\mu_i} (1 + L_i(n - 1)) = \frac{1}{\tilde{\lambda}_i} (1+ L_i(n-1)) \;\;\;\; \forall i \in S$,
\item $L_i(n) = \frac{n \pi_i W_i(n)}{\sum_{j \in \mathcal{N}} \pi_j W_j(n)} \;\;\;\; \forall i \in \mathcal{N}$,
\end{itemize}
where $n$ ranges from 1 to $m$. 

Finally, the throughput (or mean arrival rate) to each station is given by Little's theorem \cite[p.152]{Bertsekas1992}: $\Lambda_i(m) = L_i(m)/W_i(m)$ for all $i \in S$. Combining equations \eqref{eq:gamma}, \eqref{eq:throughput} and \eqref{eq:Ai}, one readily obtains the availability at each station as $A_i(m) = \Lambda_i(m)/\tilde{\lambda}_i$.

This procedure scales well to a large number of stations and vehicles, and is applied in Section \ref{sec:fleet} to real-world settings involving hundreds of stations and thousands of vehicles, to assess the potential performance of an autonomous MOD system in New York City. 

The rebalancing promoting policy considered so far, while providing useful insights into the performance and operations of an autonomous MOD system, is ultimately an open-loop policy and hence of limited applicability.  In the next section, we use insights gained from the ORP to formulate a \emph{closed-loop} rebalancing policy for the robotic vehicles that appears to perform well in practice.

\subsection{Real-time rebalancing policy} \label{sec:realtime}
In this section, we introduce a practical real-time rebalancing policy that can be implemented on real autonomous MOD systems. In reality, customers arriving at a station would wait in line rather than leave the system immediately (as in the loss model) if a vehicle is not available. In the mean time, information could be collected about the customer's destination and used in the rebalancing process. Let $v^{\text{own}}_i(t)$ be the number of vehicles ``owned'' by station $i$, that is, vehicles that are at station $i$, on their way to station $i$, or will be on their way to station $i$. We can write
$v^{\text{own}}_i(t) = v_i(t) + \sum_j v_{ji}(t) + \sum_j \tilde{c}_{ji}(t)$, where $v_i(t)$ is the number of vehicles at station $i$, $v_{ji}(t)$ is the number of vehicles enroute from station $j$ to $i$, and $\tilde{c}_{ji}$ is the number of passengers at station $j$ that are about to board an available vehicle to station $i$. Note that $\sum_i \tilde{c}_{ji}(t) \leq v_j(t)$. Let $v^e_i(t):= v^{\text{own}}_i(t) - c_i(t)$ be the number of excess vehicles there will be at station $i$, where $c_i(t)$ is the number of customers at station $i$. 
%m + \sum_i \sum_j \tilde{c}_{ji}(t) - \sum_i c_i(t) = 
The total number of excess vehicles is given by $\sum_i v^e_i(t) = \sum_i \Bigl( v_i(t) + \sum_j v_{ji}(t) + \sum_j \tilde{c}_{ji}(t) - c_i(t) \Bigr) =
 m + \sum_i \min\{v_i(t), c_i(t)\} - \sum_i c_i(t) = m - \sum_i \max\{c_i(t) - v_i(t), 0\}$. The second equality replaces $v_i(t) + \sum_i v_{ji}(t)$ with the total number of vehicles and asserts that in the current time step, either all of the customers or all the vehicles will leave the station. The last equality is obtained by considering both cases when $c_i(t) \geq v_i(t)$ and when $c_i(t) < v_i(t)$.

Through rebalancing, we may wish to distribute these excess vehicles evenly between all the stations, in which case each station will have no less than $v^d_i(t)$ vehicles, given by
$v^d_i(t) = \left \lfloor{(m - \sum_i \max\{c_i(t)-v_i(t),0\})/N}\right \rfloor.$
Accordingly, every $t_{\text{hor}} > 0$ time periods, the number of vehicles to rebalance from station $i$ to $j$, $\text{num}_{ij}$, is computed by solving the linear integer optimization problem
\begin{spreadlines}{-0.1em}
\begin{align*}
&\underset{\text{num}_{ij}}{\text{minimize}} && \sum_{i,j} T_{ij} \text{num}_{ij} \\
& \text{subject to} && v^e_i(t) + \sum_{j \neq i} (\text{num}_{ji} - \text{num}_{ij} ) \geq v^d_i(t) & \forall i \in S \\
& && \text{num}_{ij} \in \naturals \;\;\;\;\;\;\;\;\;\; \forall i,j \in S, i \neq j
\end{align*}
\end{spreadlines}
This rebalancing policy takes all the current information known about the system and sets the rebalancing rates (in this case, the number of rebalancing vehicles) so that excess vehicles are distributed evenly to all the stations. This is in part inspired by the optimization problem in Theorem \ref{theo:main}.  It can be shown that the constraint matrix is totally unimodular, and the problem can be solved as a linear program, as the resulting solution will be necessarily  integer-valued \cite[Section 5]{Pavone.ea:IJRR12}. The rebalancing policy presented here is closely related to the one presented in \cite{Pavone.ea:IJRR12}, the main difference being the inclusion of the current customers in line within the optimization process.

The real-time rebalancing policy will be used in section \ref{sec:fleet} to validate the vehicle availability performance criterion. 

\section{Case Study: Autonomous MOD in Manhattan} \label{sec:fleet}
In this section we apply our availability analysis using the loss model to see how many robotic vehicles in an autonomous MOD system would be required to replace the current fleet of taxis in Manhattan while providing quality service at current customer demand levels. 
In 2012, over 13,300
taxis in New York City make over 15 million trips a month or 500,000 trips a day, with around 85\% of trips within Manhattan. Our study used taxi trip data collected on March 1, 2012\footnote{The data is courtesy of the  New York City Taxi \& Limousine Commission.}
consisting of 439,950 trips within Manhattan. First, trip origins and destinations are clustered into $N = 100$ stations throughout the city using \emph{k}-means clustering. The resulting locations of the stations are such that a demand is on average less than $300$m from the nearest station, or approximately a 3-minute walk. The system parameters $\lambda_i$, $p_{ij}$, and $T_{ij}$ are estimated for each hour of the day using trip data between each pair of stations with Laplace smoothing. Some congestion effects are implicitly taken into account in the computation of $T_{ij}$, which uses the Manhattan distance and an average speed estimated from the data. 

Vehicle availability is calculated for 3 cases - peak demand (29,485 demands/hour, 7-8pm), low demand (1,982 demands/hour, 4-5am), and average demand (16,930 demands/hour, 4-5pm). For each case, vehicle availability is calculated as a function of the fleet size using MVA techniques. The results are summarized in Figure \ref{fig:availabilityPlot}.
\begin{figure}[h]
\centering
\vspace{-1.5em}
		\subfigure[]{\label{fig:availabilityPlot}	\includegraphics[width=0.24\textwidth]{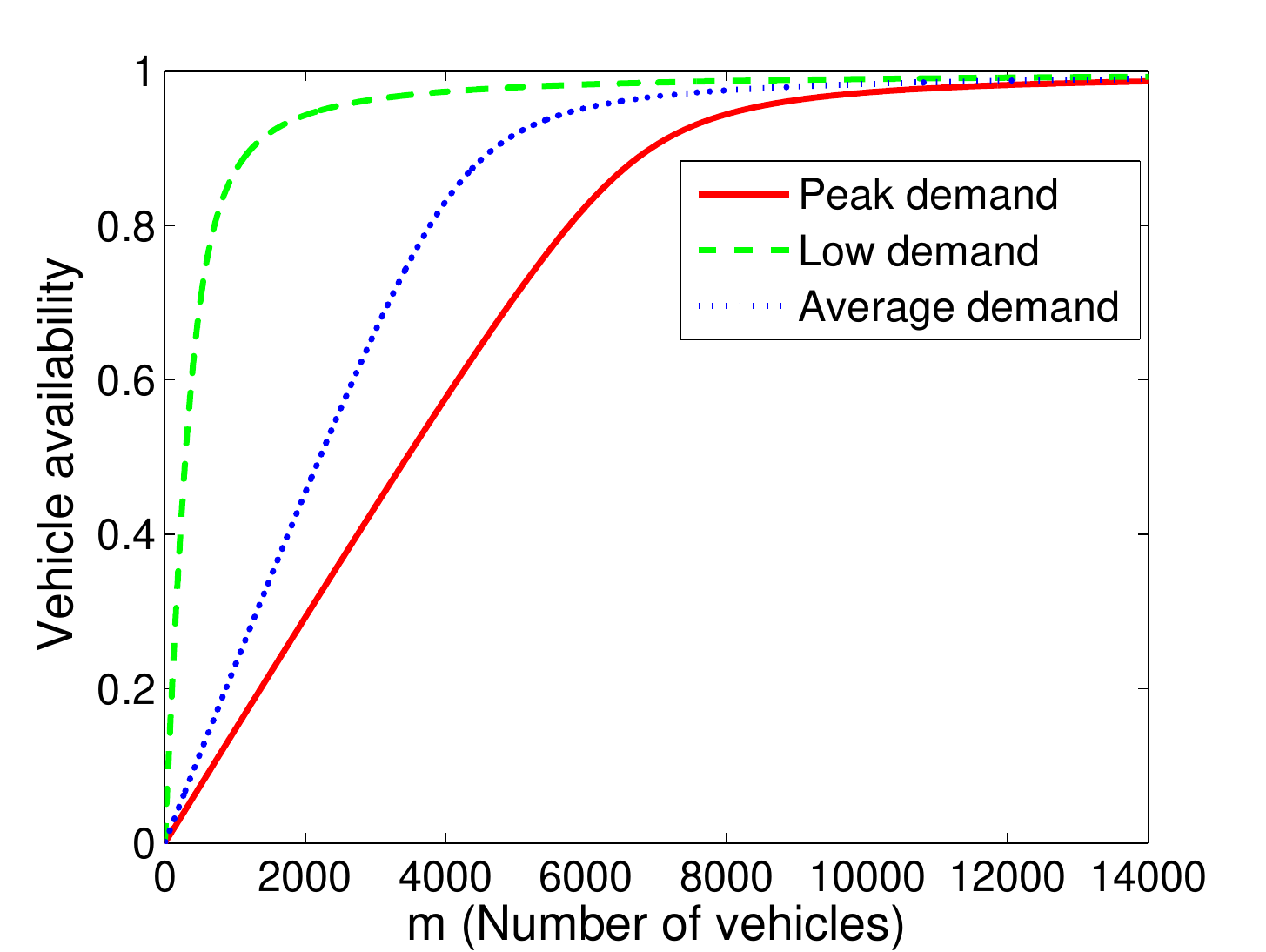}} 
		\hspace{-1em}
		\subfigure[]{\label{fig:waitTimes}
\includegraphics[width=.24\textwidth]{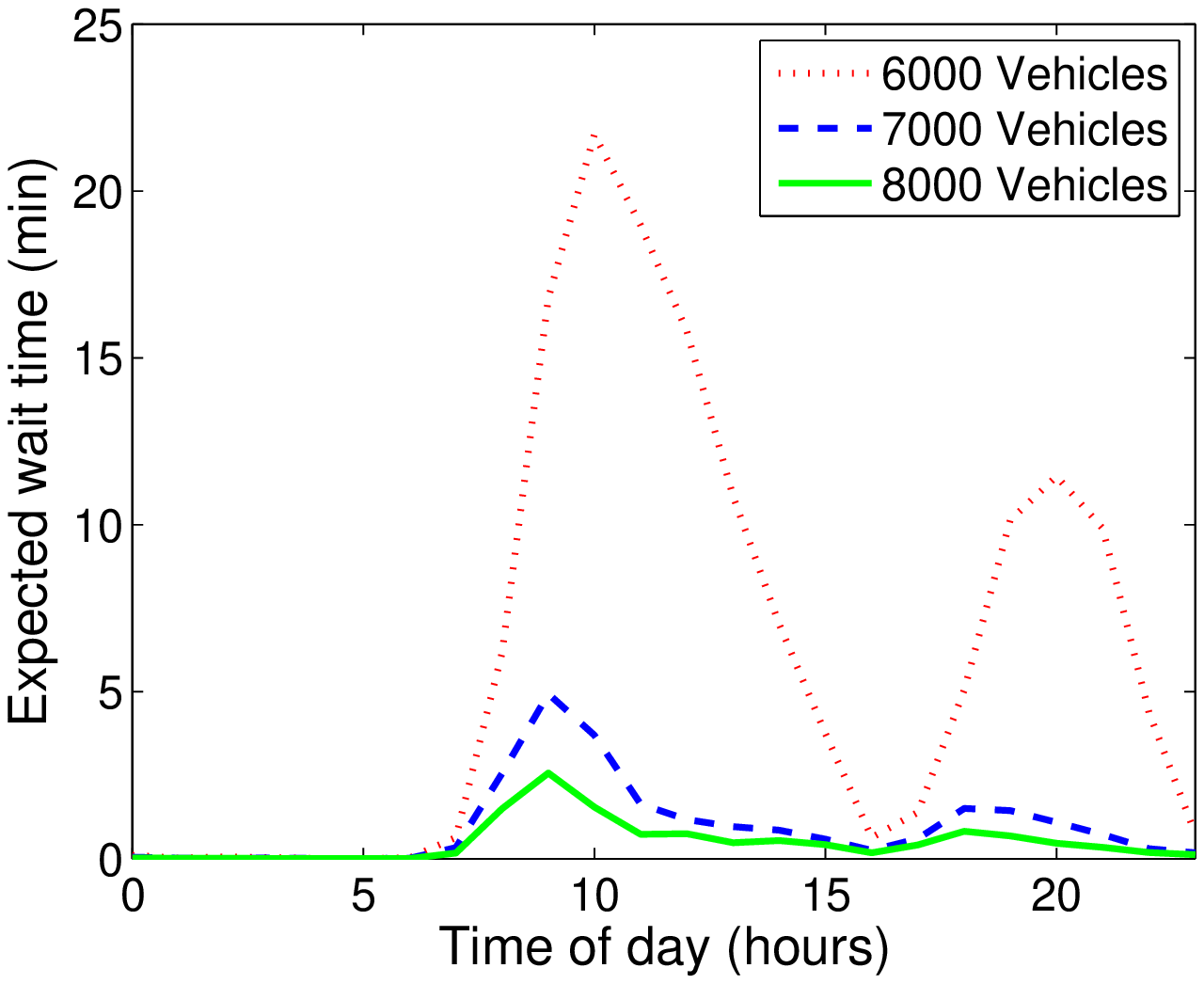}}
\vspace{-1em}
\caption{\ref{fig:availabilityPlot}: vehicle availability as a function of system size for 100 stations in Manhattan. Availability is calculated for peak demand (7-8pm), low demand (4-5am), and average demand (4-5pm). \ref{fig:waitTimes}: Average customer wait times over the course of a day, for systems of different sizes.}
\label{fig:availabilityResults}
\vspace{-0.5em}
\end{figure}

For high vehicle availability (say, 95\%), we would need around 8,000 vehicles ($\sim$60\% of the current fleet size) at peak demand and 6,000 vehicles at average demand.
This suggests that an autonomous MOD system with 8,000 vehicles would be able to meet 95\% of the taxi demand in Manhattan, assuming 5\% of passengers are impatient and are lost when a vehicle is not immediately available. However, in a real system, passengers would wait in line for the next vehicle rather than leave the system, thus it is important to determine how vehicle availability relates to customer waiting times. We characterize the customer waiting times through simulation, using the real-time rebalancing policy described in Section \ref{sec:realtime}. Figure \ref{fig:simulation} shows a snapshot of the simulation environment with 100 stations and 8,000 vehicles. Simulation are performed with discrete time steps of 2 seconds and a simulation time of 24 hours. The time-varying system parameters $\lambda_i$, $p_{ij}$, and average speed are piecewise constant, and change each hour based on values estimated from the taxi data. Travel times $T_{ij}$ are based on average speed and Manhattan distance between $i$ and $j$, and rebalancing is performed every 15 minutes. Three sets of simulations are performed for 6,000, 7,000, and 8,000 vehicles, and the resulting average waiting times are shown in Figure \ref{fig:waitTimes}.  

\begin{figure}[h]
\centering
\includegraphics[width=.33\textwidth]{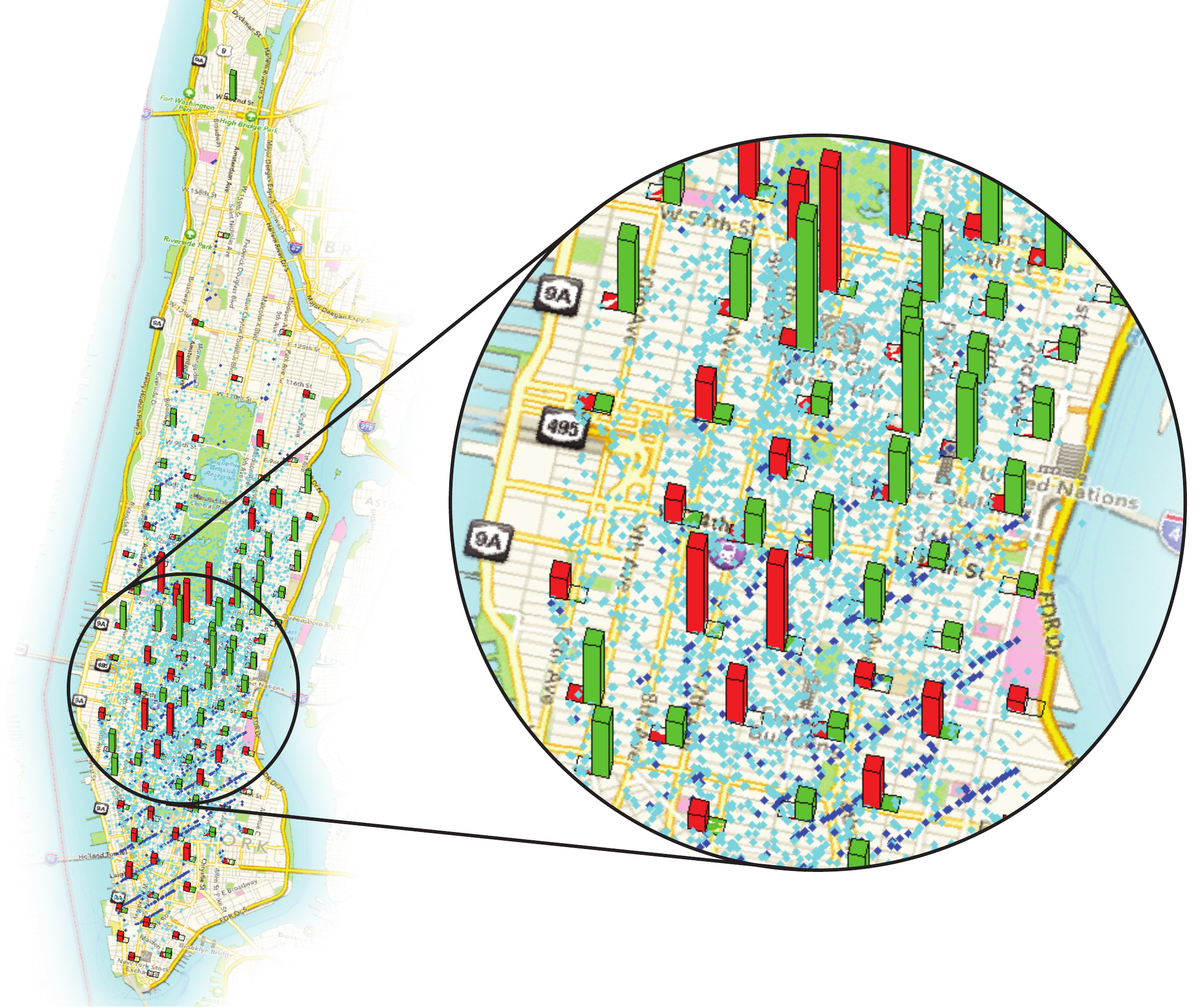}
\caption{Simulation environment with 100 stations in Manhattan. Red bars indicate waiting customers, green bars indicate available vehicles, cyan dots are vehicles travelling with passengers and blue dots are rebalancing vehicles.}
\label{fig:simulation}
\vspace{-1.5em}
\end{figure}

Figure \ref{fig:waitTimes} shows that for a 7,000 vehicle fleet, the peak averaged wait time is less than 5 minutes (9-10am) and for 8,000 vehicles, the average wait time is only 2.5 minutes. The simulation results show that high availability (90-95\%) does indeed correspond to low customer waiting time and that a autonomous MOD system with 7,000 to 8,000 vehicles (50-60\% of the size of the current taxi fleet) can provide adequate service with current taxi demand levels in Manhattan. 
\vspace{-.3em}
\section{A Mean Value Analysis Approach to the Analysis of Congestion Effects}\label{sec:cong}
The queueing model described in Section \ref{sec:model} does not consider congestion effects (roads are modeled as \emph{infinite} server queues, so the travel time for each vehicle is independent of all other vehicles).
However, if too many rebalancing vehicles travel on a route that is already congested, they can cause a traffic jam and decrease throughput in the entire system. Hence, in some scenarios, adding robotic vehicles to improve the quality of service might indeed have the opposite effect.

In this section, we propose an approach to study congestion effects that leverages our queueing-theoretical setup. The key idea is to change the infinite server road queues to queues with a \emph{finite} number of servers, where the number of servers on each road represents the \emph{capacity} of that road. This road congestion model is similar to ``vertical queueing'' models that have been used in congestion analysis for stop-controlled intersections \cite{Madanat1994} and for traffic assignment \cite{Huang2002}. In traditional traffic flow theory \cite{Lieu2003}, the flow rate of traffic increases with the density of vehicles up to a critical value at which point the flow decreases, marking the beginning of a traffic jam. 
By letting the number of servers represent the critical density of the road, the queueing model becomes a good model for traffic flow up to the point of congestion.

Remarkably, the Jackson network model presented in Section \ref{sec:model}  can be extended to the case where roads are modeled as finite-server queues; furthermore, the results presented in Section  \ref{sec:back} are equally valid. However, the travel times are no longer simply equal to the inverse of the service rates of the road queues, which significantly complicates the formulation of an analogue of problem ORP. While the issue of finding optimal rebalancing policies in the presence of congestion effects is left for future research, in this paper we show how \emph{given} a rebalancing policy one can compute performance metrics such as vehicle availability (for example, one can study the effects of congestion on the performance of the rebalancing policies considered in Section \ref{sec:rebalancing}). 

In our approach, we first model the road network as an abstract queueing network with finite-server road queues, then we apply an extended version of the MVA algorithm for finite-server queues, the details of which is presented in \cite{Reiser1980}.

\subsection{Mapping physical roads into finite-server road queues} \label{sec:mapping}
The main difficulty in mapping the capacities of the road network into the number of servers of the queueing model (or ``virtual'' capacities, denoted by $m_{ij}$) is that trips from different origins and destinations may share the \emph{same} physical road.

\begin{figure}[htbp]
\centering
\vspace{-1em}
\includegraphics[width=.25\textwidth]{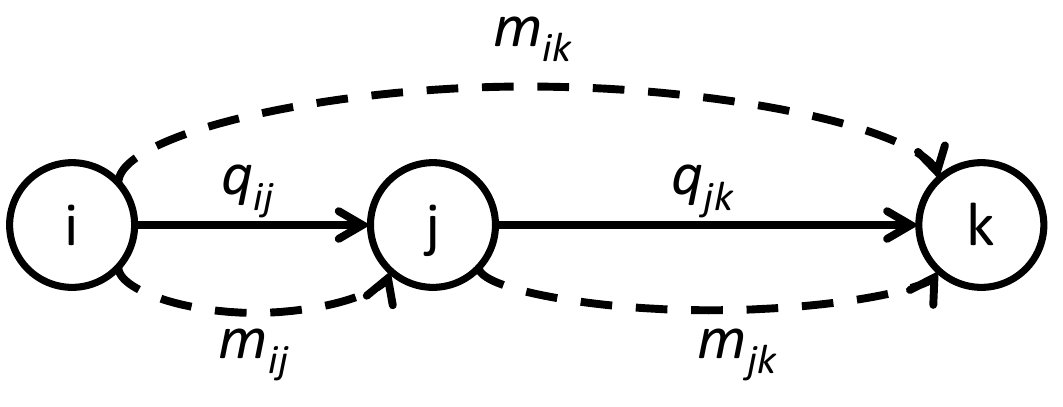}
\vspace{-1em}
\caption{A simple 3-station example showing the procedure of mapping physical roads  into finite-server road queues.}
\label{fig:threeStation}
\vspace{-0.5em}
\end{figure}
As a simple example, consider the 3-station network shown in Figure \ref{fig:threeStation}. Let $q_{ij}$ represent the maximum number of vehicles that can travel on the road between station $i$ and station $j$ without causing significant congestion. $m_{ij}$, the number of servers between $i$ and $j$, represents the number of vehicles that can travel between $i$ and $j$ before delays occur due to queueing.
In the simple network, to go from station $i$ to station $k$, one must pass through station $j$. Hence, one has the following consistency constraints
\vspace{-0.5em}
\begin{align}
m_{ij} + m_{ik}, \leq q_{ij}, \quad m_{jk} + m_{ik} \leq q_{jk}.
\label{eq:mconstraints}
\vspace{-0.5em}
\end{align} 
To maximize the overall road usage, we can define a quadratic objective that seeks to minimize the difference between the real road capacities and the sum of the virtual road capacities:
\vspace{-0.5em}
\begin{equation*}
\min_{m_{ij}, m_{jk}, m_{ik}} (m_{ij} + m_{ik} - q_{ij})^2 + (m_{jk} + m_{ik} - q_{jk})^2
\vspace{-0.5em}
\end{equation*}
However, this optimization problem, along with the constraints \eqref{eq:mconstraints}, does not yield a unique solution because nothing is assumed about the relative usage rates of the road queues. If relative road usage is known, the $m_{ij}$'s can be assigned proportional to the amount of traffic between each pair of stations that use the road. Let $\pi_{ij}$ be the relative throughput of the road queue between station $i$ and $j$, consistent with the earlier definition. \emph{Heuristically}, the throughputs $\{\pi_{ij}\}_{ij}$ may be obtained from the arrival rates and travel patterns of passengers or from the analysis of a given rebalancing policy assuming no congestion (according to the procedure discussed in Section \ref{subsec:metrics}). For the simple example, one can write 
\begin{align}
m_{ik} \leq \frac{q_{ij} \pi_{ik}}{\pi_{ik} + \pi_{ij}}, \quad m_{ik} \leq \frac{q_{jk} \pi_{ik}}{\pi_{ik} + \pi_{jk}}.
\label{eq:mconstraints2}
\end{align}
Similar constraints can be written for $m_{ij}$ and $m_{jk}$ so that \eqref{eq:mconstraints} is satisfied. 

For a general road network, let $B_{ij}$ be the set of possible non-cyclic paths from station $i$ to $j$ (assuming no back tracking) and $C_{b_{ij}}$ be the set of road segments along path $b_{ij} \in B_{ij}$. The number of possible paths from $i$ to $j$ is given by $\lvert B_{ij} \rvert$. Let $a_{ij}^c$ denote the fraction of trips from $i$ to $j$ that go through road segment $c = \{\text{origin}, \text{destination}\}$, where $c \in C_{b_{ij}}$. Denote by $q_c$ the capacity of road segment $c$. 
For trips going through multiple road segments, the virtual road capacity is determined by the segment with the lowest capacity. One can then consider as virtual road capacities:
\begin{equation*}
m_{ij} = \sum_{b \in B_{ij}} \min_{c \in C_{b_{ij}}} \Big \{\frac{q_c a_{ij}^{c} \pi_{ij}}{\sum_{k,l, \text{s.t.} c \in C_{b_{kl}} } a_{kl}^{c} \pi_{kl}} \Big \}.
\label{eq:mij}
\end{equation*}

In the next section we will apply this approach to study congestion effects for autonomous MOD systems on a very simple transportation network. 

\subsection{Numerical study of congestion effects}\label{sec:simCon}
In this section we use a simple 9-station road network (shown in Figure \ref{fig:congestion}) to illustrate the impact of rebalancing vehicles on congestion.
\begin{figure}[htbp]
\begin{minipage}{.22\textwidth}
  \centering
  \begin{subfigure}{}
    \centering
    \includegraphics[width=.65\linewidth]{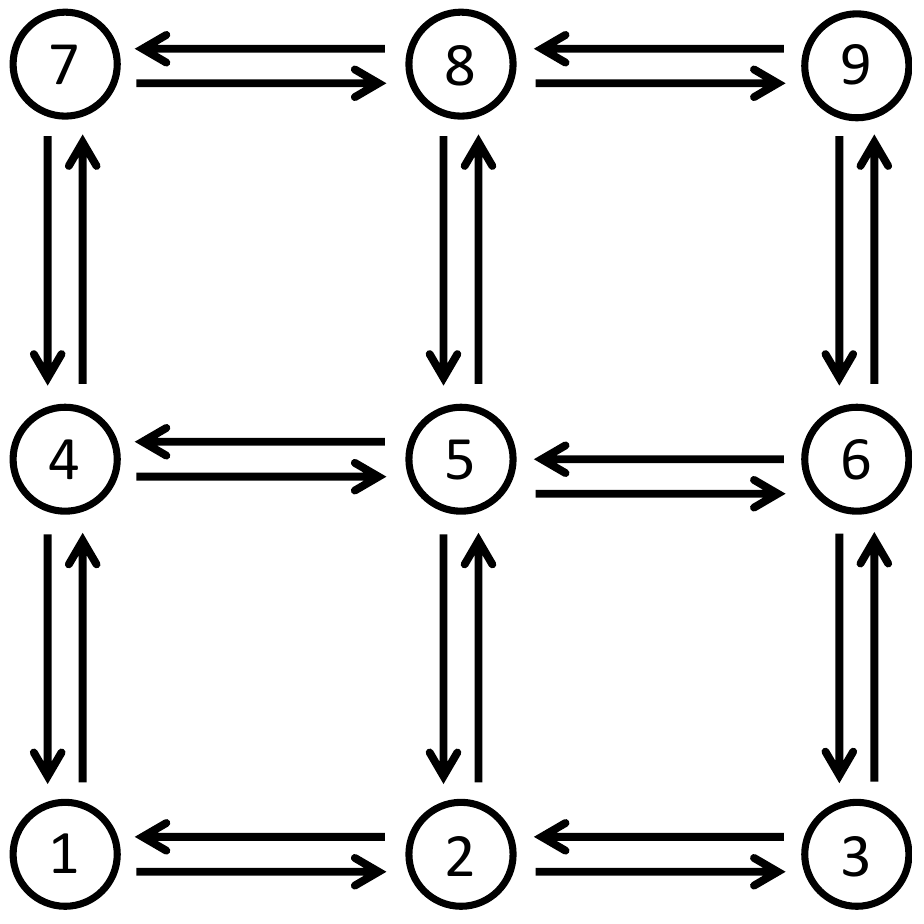}
    %\label{fig:stations9}
  \end{subfigure} \\
  \begin{subfigure}{}
      \centering
    \includegraphics[width=.4\linewidth]{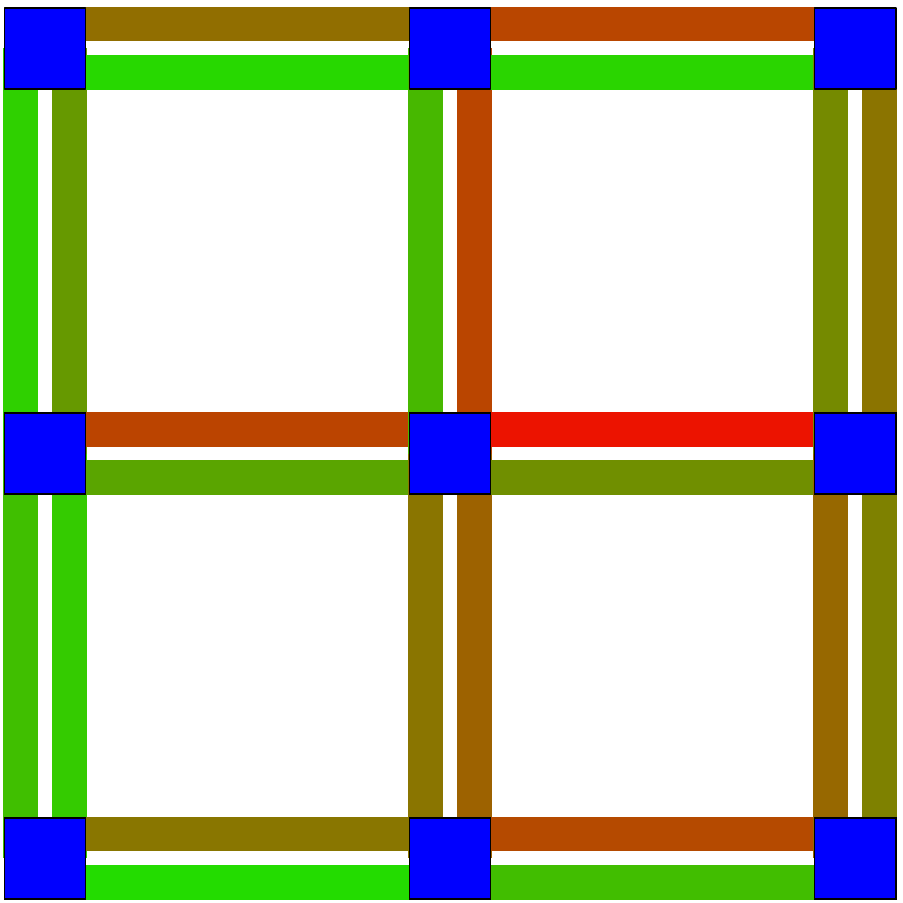}
    %\label{fig:norebalance}
  \end{subfigure} 
  \begin{subfigure}{}
      \centering
    \includegraphics[width=.4\linewidth]{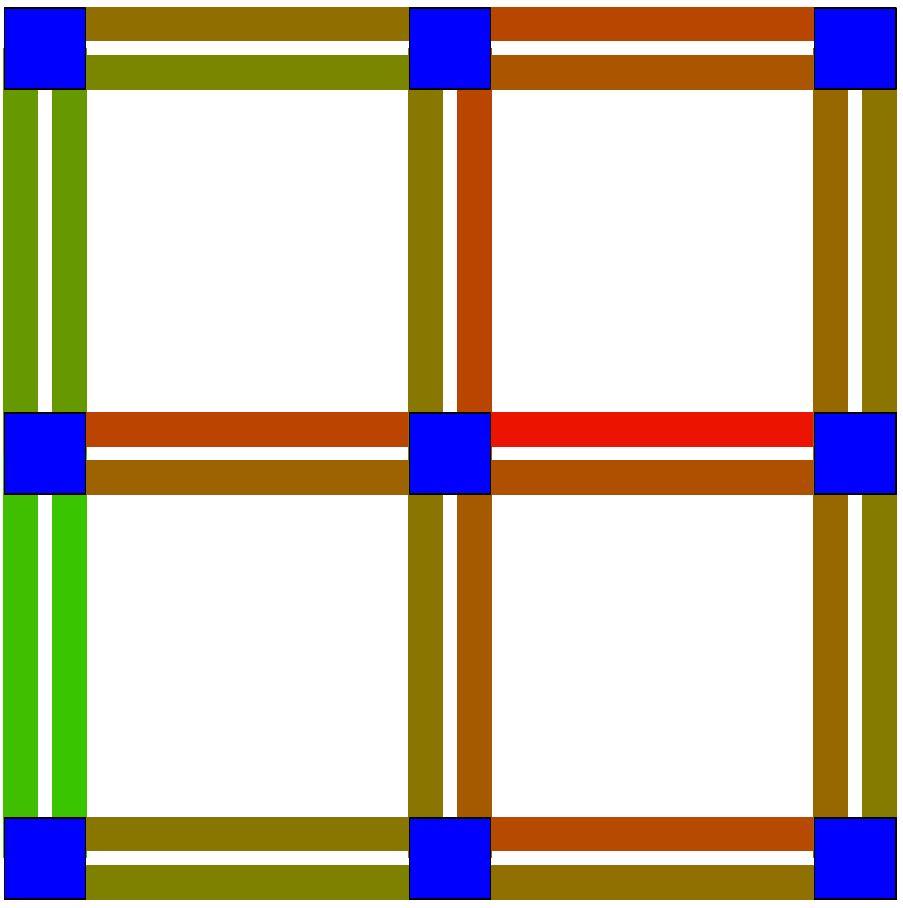}
    %\label{fig:rebalance}
  \end{subfigure}
\end{minipage}
\hspace{-1em}
\begin{minipage}{.26\textwidth}
\centering
\vspace{-3em}
  \begin{subfigure}{}
    \begin{flushright}
    \includegraphics[width=1\linewidth]{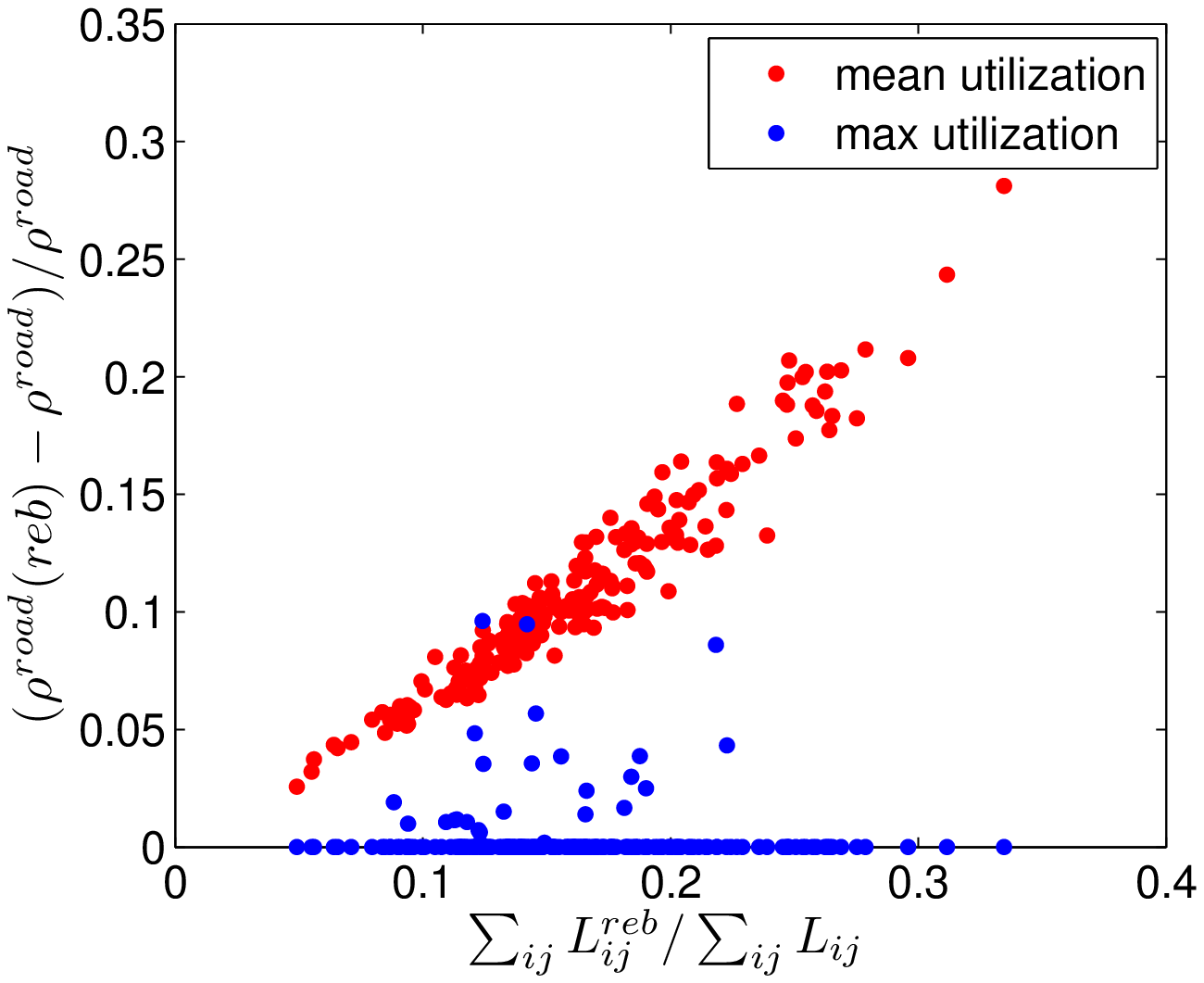}
    %\label{fig:congestionIncrease}
    \end{flushright}
  \end{subfigure}
\end{minipage}%
\caption{Top left: Layout of the 9-station road network. Each road segment has a capacity of 40 vehicles in each direction. Bottom left: The first picture shows the 9-station road network without rebalancing. The color on each road segment indicates the level of congestion, where green is no congestion, and red is heavy congestion. The second picture is the same road network with rebalancing vehicles. Right: The effects of rebalancing on congestion. The x-axis is the ratio of rebalancing vehicles to passenger vehicles on the road. The y-axis is the fractional increase in road utilization due to rebalancing.}
\label{fig:congestion}
\vspace{-0.3em}
\end{figure}
The stations are placed on a square grid, and joined by 2-way road segments each of which is 0.5 km long. Each road consists of a single lane, with a critical density of 80 vehicles/km. \footnote{If each vehicle is 5 m, this critical density represents a vehicle-to-vehicle separation of 1.5 car-lengths.} This means that the capacity of each road segment $c$ is $q_c = 40$ vehicles. Each vehicle travels at 30 km/h (8.33 m/s) in free flow, which means the travel time along each road segment is 1 minute in free flow. 

To gain insight into the general system behavior, a variety of systems with different levels of imbalance must be studied. First, arrival rates and routing distributions are randomly generated and rebalancing rates are computed using \eqref{eq:opt2}. In steady state, the fraction of vehicles in each road queue $ij$ is given by $\pi_i \tilde{p}_{ij}$ (Lemma \ref{lemma:folding}). If we assume 100\% availability ($A_i = 1$), the expected rate of vehicles entering each road queue is given by
$\Lambda_{ij} = \lambda_i p_{ij}$. Using Little's theorem, the expected number of vehicles on each road queue is given by $L_{ij} = \Lambda_{ij} T_{ij}$. 
The availability assumption can be justified by the fact that a real system would operate within the regime of high availability and that the number of vehicles on the road gets very close to $L_{ij}$ as availability increases. Similarly, the expected number of rebalancing vehicles on each road queue is given by $L^{\text{reb}}_{ij} = \beta_{ij} T_{ij}$.

 To map the queueing network onto the road network, we adopt a similar procedure as the one used to estimate $m_{ij}$ in section \ref{sec:mapping}. Recall that $B_{ij}$ is the set of paths from station $i$ to station $j$. We adopt the routing strategy that uniformly distributes vehicles from $i$ to $j$ along each path $b_{ij} \in B_{ij}$. The number of vehicles that go through each road segment, $L^{\text{road}}_{c}$, is then the sum of the number of vehicles from each station to every other station that pass through the road segment, given by
$L^{\text{road}}_c = \sum_{i,j, \text{s.t.} \,  c \in C_{b_{ij}}} L_{ij} / \lvert B_{ij} \rvert$.
Note that for stability, $L^{\text{road}}_{c} < q_c$. The road utilization is given by $\rho^{\text{road}}_c = L^{\text{road}}_{c}/q_c$. 

Figure \ref{fig:congestion} plots the vehicle and road utilization increases due to rebalancing for 500 randomly generated systems. The x-axis shows the ratio of rebalancing vehicles to passenger vehicles on the road, which represents the inherent imbalance in the system. The red data points represent the increase in average road utilization due to rebalancing and the blue data points represent the utilization increase in the most congested road segment due to rebalancing. It is no surprise that the average road utilization rate is a linear function of the number of rebalancing vehicles. However, remarkably, the maximum congestion increases are much lower than the average, and are in most cases, zero. This means that while rebalancing generally increases the number of vehicles on the road, rebalancing vehicles mostly travel along less congested routes and rarely increase the maximum congestion in the system. This can be seen in Figure \ref{fig:congestion} bottom left, where rebalancing clearly increases the number of vehicles on many roads but not the most congested road segment (from station 6 to station 5). 

In a few rare cases, the maximum congestion in the system is increased up to 10\%. This may cause heavy congestion in systems where congestion is already prevalent (say, $>$90\%). In these cases, an intelligent routing strategy becomes crucial. While uniform routing along different paths helps distribute vehicles throughout the road network, a better routing strategy would actively route vehicles away from congested roads and perhaps even limit rebalancing when it may cause further delays. This is related to the simultaneous departure and routing problem \cite{Huang2002}, a class of dynamic traffic assignment (DTA) problems, and will be the subject of future work. 

\section{Conclusions}\label{sec:conc}

In this paper we presented and analyzed a queueing-theoretical model for autonomous MOD systems. We showed that an optimal open-loop policy can be readily found by solving a linear program. Based on this policy, we developed a closed-loop, real-time rebalancing policy that appears quite efficient, and we applied it to a case study of New York City. Finally, we showed that vehicle rebalancing can have a detrimental impact on traffic congestion in already-congested systems but in most cases, rebalancing vehicles tend to travel along less congested roads.

This paper leaves numerous important extensions open for further research. First, it is of interest to develop rebalancing policies that can both route rebalancing vehicles along less congested roads and limit the number of rebalancing vehicles when the system is overly congested. Second, we plan to study different performance metrics (e.g., minimization of waiting times) and include a richer set of constraints (e.g., time windows to pick up the customers). Third, it is of interest to include in the model the provision of mass transit options (e.g., a metro) and develop optimal coordination algorithms for such an intermodal system. Fourth, we plan to consider additional case studies (e.g., from Asia and Europe) and study in more details the economic and societal benefits of robotic MOD systems. Finally, we plan to demo the algorithms on real driverless vehicles providing MOD service in a gated community.

\clearpage
\bibliographystyle{plainnat}
\bibliography{MODbib}

\clearpage
\section*{Supplemental Material}

\begin{proof}[Proof of Lemma \ref{lemma:folding}]
For each node $i \in \mathcal N$ equation \eqref{eq:traffic} can be separated into SS nodes and IS nodes as follows
\begin{align*}
\pi_i &= \sum_{j \in \mathcal{N}} \pi_j \tilde{r}_{ji} = \sum_{j \in S} \pi_j \tilde{r}_{ji} + \sum_{j \in I} \pi_j \tilde{r}_{ji}.
\end{align*}

Consider, first, an SS node, i.e., consider $i$ in $S$. Then one can write,
\begin{align*}
\pi_i &= \underbrace{\sum_{j \in S} \pi_j \tilde{r}_{ji}}_{=0} + \sum_{j \in I} \pi_j \tilde{r}_{ji} = \sum_{\substack{j \in I \\ i = \text{Child}(j)}} \pi_j,
\end{align*}
where $\sum_{j \in S} \pi_j \tilde{r}_{ji} = 0$ since SS nodes are connected exclusively by IS nodes. The last equality follows from the fact that whenever a child node of an IS node $j$ is the SS node $i$, then $\tilde{r}_{ji} = 1$. 

Consider, now, an IS node, i.e., consider $i$ in $I$. Let   $\text{Parent}(i) = k$ and $\text{Child}(i) = l$. Then one can write,
\begin{align*}
\pi_i &= \sum_{j \in S} \pi_j \tilde{r}_{ji} + \underbrace{\sum_{j \in I} \pi_j \tilde{r}_{ji}}_{=0} = \pi_k \tilde{p}_{kl},
\end{align*}
where $\sum_{j \in I} \pi_j \tilde{r}_{ji} = 0$ since IS nodes are connected \emph{exclusively} to SS nodes, and the second equality follows from the fact that 
a single SS node feeds into each IS node with probability $\tilde{p}_{kl}$. This proves the second claim. 

Collecting the results so far, we obtain, for each $i$ in $S$,
\begin{equation*}
\begin{split}
\pi_i &= \sum_{\substack{j \in I \\ i = \text{Child}(j)}} \pi_j =\sum_{\substack{j \in I \\ i = \text{Child}(j)}} \pi_{\text{Parent(j)}}\, \, \tilde{p}_{\text{Parent(j)}\, i} =  \sum_{k \in S} \pi_k \tilde{p}_{ki},
\end{split}
\end{equation*}
which proves the first claim.
\end{proof}

\begin{proof}[Proof of Lemma \ref{lemma:gamma}]
Let us prove the first part of the lemma.
By assumption, the probabilities $\{p_{ij}\}_{ij}$ constitute an irreducible Markov chain. By equation \eqref{eq:prob}, the probabilities $\{\tilde{p}_{ij}\}_{ij}$ lead to an irreducible Markov chain as well. The $\pi$ vector satisfying equation \eqref{eq:traffic2} is the steady state distribution for the transition probabilities $\{\tilde{p}_{ij}\}_{ij}$ and by the Perron-Frobenius theorem, it is positive \cite[p.673]{Meyer2000}. In other words, $\pi_i > 0$ for all $i \in S$. By the definition of the relative utilizations $\gamma_i$ (see equation \eqref{eq:gamma}), we obtain the first part of the claim.
 
Let us now consider the second part of the lemma. Recall that, by assumption, $p_{ii} = 0$ and $\alpha_{ii} = 0$. By Lemma \ref{lemma:folding}, for any $i \in S$, one can write
\begin{align*}
\pi_i &= \sum_{j \in S} \pi_j \tilde{p}_{ji} \\
&= \sum_{j \in S} \pi_j \Big(\alpha_{ji} p_j + p_{ji}(1-p_j) \Big) \\
&= \sum_{j \in S} \pi_j \Big( \alpha_{ji} \frac{\psi_j}{\lambda_j + \psi_j} + p_{ji} \frac{\lambda_j}{\lambda_j + \psi_j} \Big) \\
&= \sum_{j \in S} \frac{\pi_j}{\lambda_j + \psi_j} (\alpha_{ji} \psi_j + p_{ji} \lambda_j) \\
&= \sum_{j \in S} \gamma_j (\alpha_{ji} \psi_j + p_{ji} \lambda_j),
\end{align*}
where the second equality follows from equation \eqref{eq:prob}, the third equality follows from equation \eqref{eq:Bernoulli}, and the last equality follows from equations \eqref{eq:gamma}, \eqref{eq:muij}, and  \eqref{eq:lgen}. This concludes the proof.

\begin{equation*}
(\lambda_i + \psi_i) \gamma_i = \sum_{k \in S} \gamma_k (\alpha_{ki} \psi_k + p_{ki} \lambda_k).
\end{equation*}
By applying the constraint $\gamma_i = \gamma_k$ and dividing both sides by $\gamma_i$ one obtains 
\begin{equation*}
\lambda_i + \psi_i = \sum_{k \in S} \alpha_{ki} \psi_k + p_{ki} \lambda_k.
\end{equation*}
\end{proof}

\end{document}